\documentclass[wcp]{jmlr}

\usepackage{longtable}

\usepackage{booktabs}

\usepackage{lineno}
\usepackage{enumitem}
\usepackage{silence}
\usepackage[american]{babel}

\usepackage{booktabs} 
\usepackage{tikz} 
\usepackage{bm}
\usepackage{graphicx}
\usepackage{float}
\usepackage{amssymb,amsfonts}%
\usepackage{multirow}%
\usepackage{makecell}
\usepackage{arydshln}
\usepackage{bbm}%

\RestyleAlgo{ruled}
\DontPrintSemicolon
\LinesNumbered
\SetAlgoVlined
\SetKwInOut{Input}{Input}
\SetKwInOut{Output}{Output}

\makeatletter
\let\Ginclude@graphics\@org@Ginclude@graphics 
\makeatother

\title[]{Learning Sparsest Linear Causal DAGs with Latent Confounders via Higher-Order Cumulants}

\author{
 \Name{Ming Cai}* \Email{cai.ming.52d@st.kyoto-u.ac.jp}\\
 \addr Graduate School of Informatics, Kyoto University, Kyoto, Japan
 \AND
 \Name{Hisayuki Hara} \Email{hara.hisayuki.8k@kyoto-u.ac.jp}\\
 \addr Institute for Liberal Arts and Sciences, Kyoto University, Kyoto, Japan
}

\begin{document}

\makeatletter
\let \@jmlrpages \@empty
\makeatother

\maketitle

\begin{abstract}
Recovering the exact directed acyclic graph (DAG) in linear non-Gaussian acyclic models with latent confounders (LvLiNGAM) remains a challenging problem. Although LvLiNGAM is identifiable only up to an observational equivalence class, each equivalence class is characterized by a unique sparsest DAG. Recovering the sparsest DAG from finite samples, however, remains difficult. Although existing methods are asymptotically consistent, they do not provide an explicit finite-sample procedure for recovering the unique sparsest DAG, nor do they handle models with an arbitrary number of latent confounders.

In this paper, we propose a finite-sample method for recovering the sparsest DAG without imposing any restriction on the number of latent confounders. Simulation studies and real-data analyses demonstrate that the proposed method achieves superior finite-sample performance compared with existing approaches. 
\end{abstract}

\section{Introduction}\label{sec:intro}
Linear non-Gaussian acyclic models (LiNGAMs) provide a powerful framework for causal discovery 
\citep{Shimizu2006, Shimizu2011}.
In the absence of latent variables, LiNGAM enables complete identification of causal DAGs.
In many practical applications, however, latent confounders are unavoidable.
\citet{hoyer2008estimation} introduced LiNGAM with latent variables (LvLiNGAM) and demonstrated that any LvLiNGAM can be transformed into a canonical model in which all latent variables are mutually independent and causally precede the observed ones. 
They also estimated the mixing matrix using overcomplete independent component analysis (OICA; e.g., \citealp{eriksson2004identifiability}), assuming that the number of latent variables is known a priori.
However, because it relies on OICA, this approach is prone to converge to local optima~\citep{shimizu14aBayesianestimation}.

To avoid relying on OICA, several methods have been proposed to estimate 
canonical LvLiNGAMs via residual independence tests,
such as Pairwise LvLiNGAM \citep{Entner2011}, ParceLiNGAM \citep{tashiro2014parcelingam}, Repetitive Causal Discovery (RCD) \citep{Maeda2020, Maeda2022, maeda2022rcd}, and BANG \citep{wang2023BANG}.
However, none of these methods can fully identify ancestral relationships or parent-child relationships between observed variables that form the bow structures~\citep{wang2023BANG}. 

When bow structures are present, \citet{chen2024identification} use cumulants of observed variables to identify their ancestral relationships in the bivariate setting with a single latent confounder. 
Building on this, \citet{chen2025identification} extend the approach to multiple latent confounders in the observed bivariate case. 

\cite{schkoda2024causal} proposed ReLVLiNGAM, a recursive cumulant-based method that accommodates multiple observed variables and latent confounders. Without relying on OICA or requiring prior knowledge of the number of latent variables, ReLVLiNGAM recovers the observational equivalence class of a canonical LvLiNGAM. 

Under the genericity assumption described below, the sparsest DAG within an observational equivalence class is uniquely determined. 
The sparsest DAG is generic within its model class, whereas any denser DAG in the same observational equivalence class requires non-generic parameter values. 
This provides a natural justification for treating the sparsest DAG as the canonical representative of the observational equivalence class. 
Although ReLVLiNGAM consistently estimates the mixing matrix, it does not provide a procedure for estimating the sparsest DAG from finite samples. 

Moreover, 
the original ReLVLiNGAM requires that, in each iteration, 
an observed variable with no observed parents (an observed source hereafter) has fewer latent parents than observed siblings; otherwise, the cumulant updates and mixing-matrix estimation may fail (Figure~\ref{fig: equivalence class of IV graph}; see Appendix~\ref{sec: simple improvement}). We refer to this issue as ReLVLiNGAM’s local restriction.

In this paper, we develop a finite-sample algorithm for recovering the sparsest DAG in the observational equivalence class of canonical LvLiNGAMs without the local restriction. The proposed method builds on the top-down framework of ReLVLiNGAM. It successively infers parent-child relationships among observed variables, starting from observed sources, to recover the sparsest DAG. Within this framework, we introduce two key innovations.

First, we introduce an update rule that directly residualizes the observed variables, rather than recursively updating higher-order cumulants as in ReLVLiNGAM. Updating the observed variables rather than their cumulants mitigates the recursive propagation of errors in higher-order cumulant estimates, thereby improving finite-sample performance. Moreover, this update rule eliminates the need for the local restriction in ReLVLiNGAM, extending the applicability of the method.

Second, we introduce a sequential procedure for identifying exact parent--child relationships between each observed source and its descendants from the estimated ancestral structure. This procedure enables direct recovery of the sparsest DAG from finite samples. 

Our main contributions are as follows: (1) we propose a finite-sample algorithm for recovering the sparsest DAG in the observational equivalence class of canonical LvLiNGAMs; (2) we introduce an update rule that directly residualizes the observed variables instead of recursively updating higher-order cumulants, thereby removing the local restriction and improving finite-sample performance; (3) we propose a new parent--child criterion that enables direct recovery of the sparsest DAG from finite samples; (4) experiments on synthetic and real data demonstrate the effectiveness of the proposed method, particularly when ReLVLiNGAM's local restriction is violated.

\section{Preliminaries}
\label{sec: problem definition}
\subsection{Canonical LvLiNGAM}
\label{sec: lvLiNGAM}
Let $\bm{X}=(X_{1},\ldots,X_{p})^{\top}$ be the observed variables and $\bm{L}=(L_{1},\ldots,L_{q})^{\top}$ be the latent confounders. 
Denote a causal DAG by $\mathcal{G}=(\bm{V}, E)$, where $\bm{V}=\bm{X}\cup\bm{L}$ and $E \subset \bm{V} \times \bm{V}$ is the set of directed edges. 
Define $E^O= (\bm{X} \times \bm{X}) \cap E$ and $E^{OL}=(\bm{L} \times \bm{X}) \cap E$. 
We call $\mathcal{G}^O=(\bm{X},E^O)$ the observed DAG of $\mathcal{G}$, and
let $\mathcal{G}^{OL}=(\bm{L},\bm{X},E^{OL})$ be the latent-to-observed bipartite graph of $\mathcal{G}$.
For $X_i \in \bm{X}$, let $\mathrm{Anc}(X_i)$, $\mathrm{Des}(X_i)$, $\mathrm{Pa}(X_{i})$, and $\mathrm{Ch}(X_i)$ denote the sets of ancestors, descendants, parents, and children of $X_i$, respectively. 
For two variables $V,V' \in \bm{V}$, 
we denote a directed edge from $V$ to $V'$ by $V\to V'$, and let $\mathcal{P}(V,V')$ be the set of directed paths from $V$ to $V'$.
For two observed variables $X_i$ and $X_j$, we define their (possibly latent) confounders as
\begin{align*}
\mathrm{Conf}(X_{i}, X_{j})
= \Big\{ V: \exists\, \pi \in \mathcal{P}(V,X_i), 
\exists\, \pi' \in \mathcal{P}(V,X_j)  
\text{ s.t. } \big(\pi \cap \pi'\big) \setminus \{V\} = \emptyset
\Big\}.
\end{align*}

In this paper, we employ the canonical LvLiNGAM~\citep{hoyer2008estimation}, 
where latent variables are mutually independent, and each has at least two observed children. 
The model 
defined by $\mathcal{G}$ is expressed as 
\begin{align}
\label{eq: LvLiNGAM}
\bm{X} = \bm{\Lambda}\bm{L} + \bm{B}\bm{X} + \bm{e}, 
\end{align}
where $\bm{e} = (e_1,\ldots,e_p)^\top$, or, equivalently,
\begin{align}
\label{eq: LvLiNGAM 2}
    \bm{X} &= \left[\left(\bm{I} - \bm{B}\right)^{-1}\bm{\Lambda}, \;\; \left(\bm{I} - \bm{B}\right)^{-1}\right]
    \left[\bm{L}^{\top}, \bm{e}^{\top}
    \right]^{\top}. 
\end{align}
The components of $\bm{u} = \left(\bm{L}^{\top}, \bm{e}^{\top}\right)^{\top}$ are mutually independent, and each follows a continuous non-Gaussian distribution with nonzero higher-order cumulants.
$\bm{\Lambda}=\{\lambda_{ji}\}$ and $\bm{B}=\{b_{ji}\}$ collect the direct causal coefficients for $L_i\to X_j \in E$ and $X_i \to X_j \in E$, respectively, and $\bm{M}=\left[\left(\bm{I} - \bm{B}\right)^{-1}\bm{\Lambda}, \;\; \left(\bm{I} - \bm{B}\right)^{-1}\right]$ is the mixing matrix whose entries are total effects. 
Let $m^{OL}_{ji}$ and $m^{O}_{ji}$ denote the total effects of $L_i$ and $X_i$ on $X_j$, respectively. Since latent scales are arbitrary, without loss of generality, we fix $\lambda_{ji}=1$ for $X_j \in \mathrm{Ch}(L_i)$ with the highest causal order. The same normalization is also used in~\cite{schkoda2024causal}.

Throughout this paper, we assume that the coefficients in \eqref{eq: LvLiNGAM} and the higher-order cumulants of $\bm{u}$ are generic. We call this assumption the genericity assumption. 
In particular, the results in this paper hold except for a set of coefficients and cumulants of Lebesgue measure zero.

We denote by $P(\bm{V})$ the joint distribution of $\bm{V}$. 
Suppose that, in \eqref{eq: LvLiNGAM 2}, there exist $X_i$ and $L_j$ such that $\mathrm{Des}(L_j)=\mathrm{Des}(X_i)\cup\{X_i\}$.
\cite{salehkaleybar2020learning} showed that, under this condition, exchanging the columns of the mixing matrix $\bm{M}$ corresponding to $L_j$ and $e_i$ yields another valid LvLiNGAM representation while preserving both $P(\bm{V})$ and the ancestral relationships of $\bm{X}$. Depending on the graph structure, the resulting representation may correspond to a different causal DAG. Following their terminology, we refer to this column-exchange operation as a swap between $L_j$ and $e_i$.

Figure~\ref{fig: equivalence class of IV graph} illustrates such an example. DAG~(a) is the original causal DAG $\mathcal{G}$, whereas DAG~(b) is obtained by a swap between $L_3$ and $e_2$. Although the two DAGs induce the same $P(\bm{V})$ and ancestral relationships, the swap introduces an additional directed edge $X_1 \to X_3$ in DAG~(b). This example shows that multiple causal DAGs may belong to the same observational equivalence pattern.

Nevertheless, the sparsest DAG within an observational equivalence class is of particular interest. In the example above, DAG~(a) is sparser than DAG~(b). Moreover, under the genericity assumption, a generic parameterization of DAG~(a) corresponds to a measure-zero subset of the parameter space of DAG~(b). Thus, the sparsest DAG can be regarded as a canonical representative of the observational equivalence class.

\citet{schkoda2024causal} showed that the mixing matrix of a canonical LvLiNGAM is identifiable under the genericity assumption, which in turn identifies the corresponding observational equivalence class. Hence, the sparsest DAG within the class is also identifiable. However, ReLVLiNGAM focuses on consistent estimation of the mixing matrix, rather than direct recovery of the sparsest DAG from finite samples. Moreover, it requires that, at each iteration, every observed source has fewer latent parents than each of its observed siblings, a condition that we refer to as the local restriction. 
The model shown in Figure~\ref{fig: equivalence class of IV graph} violates this restriction, and therefore ReLVLiNGAM cannot be applied directly.

\begin{figure}[t]
    \centering
    \subfigure[\centering]{
        \includegraphics[width=0.25\textwidth]{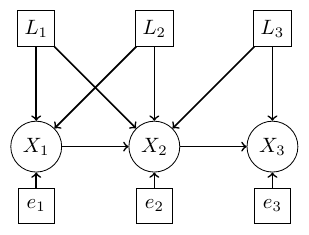}
    }
    \subfigure[\centering]{
    \includegraphics[width=0.25\textwidth]{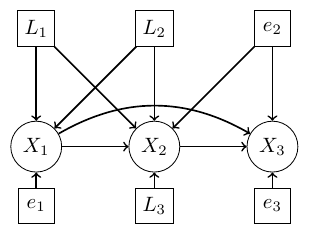}
    }
    \caption{An example of observationally equivalent models with different DAGs.}
    \label{fig: equivalence class of IV graph}
\end{figure}
\subsection{Cumulants}
\label{sec: cumulants}
In this section, we review several results established by \cite{schkoda2024causal} concerning the relationship between higher-order cumulants of LvLiNGAM variables and total effects, which also play a central role in this paper.
\begin{definition}[Cumulants \citep{Brillinger}]
\label{def: cumulant}
Let $\mathcal{I}\!=\![p]$ be the set of indices. 
For a $p$-dimensional variable $\bm{X} \!=\!(X_{1}, \!\dots, \!X_{p})^\top$,  
the $k$-th order cumulant $c^{(k)}_{i_{1}, \dots, i_{k}}$ is defined by
\begin{align*}
    c^{(k)}_{i_{1}, \dots, i_{k}}
    =\sum_{D_{i} \in \{D_{1}, \dots, D_{h}\}} (-1)^{h-1}(h-1)!\prod_{d_{i}\in D_{i}}\mathbb{E}\left[\prod_{j\in d_{i}}X_{j}\right],
\end{align*}
where 
$\{i_{1}, \dots, i_{k}\} \in 
\mathcal{I}^{k}$
and $\{D_{1}, \dots, D_{h}\}$ is the set of all partitions of $\{i_{1}, i_{2}, \dots, i_{k}\}$.
\end{definition}
When $i_{1} = i_{2} = \cdots=i_{k} = i$, $\kappa^{(k)}(X_{i})$ denotes the $k$-th order cumulant of $X_{i}$.
From \eqref{eq: LvLiNGAM 2}, 
the $k$-th order cumulants of observed variables can be rewritten as
\begin{align}\label{eq: cumulant linear}
    c^{(k)}_{i_{1}, \dots, i_{k}}
    =
    \sum_{h =1}^q m^{OL}_{i_{1}h} \dots m^{OL}_{i_{k}h}\kappa^{(k)}(L_{h}) + \sum_{h=1}^p m^{O}_{i_{1}h} \dots m^{O}_{i_{k}h}\kappa^{(k)}(e_{h}).
\end{align}
Fix two observed variables $X_i$ and $X_j$, and treat all variables in $\bm{V}\setminus \{X_i,X_j\}$ as latent.
Applying Algorithm~A in \citet{hoyer2008estimation}, we obtain a canonical LvLiNGAM in which the latent confounders $\mathrm{Conf}(X_i,X_j)=\{L^\prime_{1}, L^\prime_{2}, \cdots, L^\prime_{\ell}\}$ are mutually independent.
Without loss of generality, assume $X_j\notin\mathrm{Anc}(X_i)$. 
Then, the canonical LvLiNGAM is expressed as
\begin{align}
\label{eq: two variables}
X_{i} &\!=\!\sum_{h=1}^{\ell} {m^{OL'}_{ih}}L'_{h}\!+\!v_{i}, \quad X_{j}\!=\!\sum^{\ell}_{h = 1} {m^{OL'}_{jh}}L'_{h}\!+\!m^{O}_{ji}v_{i}\!+\!v_{j},  
\end{align}
where $v_{i}$ and $v_{j}$ are disturbances, and ${m^{OL'}_{ih}}\text{ and }{m^{OL'}_{jh}}$ are total effects from $L'_{h}$ to $X_{i}$ and $X_{j}$, respectively, in the canonical model over $X_{i}$ and $X_{j}$. 
$\ell$ is the number of confounders between $X_i$ and $X_j$ in the original model.

\cite{schkoda2024causal} estimate the canonical LvLiNGAM under the genericity assumption using higher-order cumulants of observed variables. 
They define the matrix $A^{(k_{1}, k_{2})}_{{j}, {i}}$ in \eqref{eq: a matrix} with $k_{1}< k_{2}$.  
\begin{align}\label{eq: a matrix}
    A^{(k_{1}, k_{2})}_{{j}, {i}} =
    \left[
    \begin{array}{cccc}
       c^{(k_{1})}_{i, i, \dots, i}  & c^{(k_{1})}_{i, i, \dots, j} & \dots & c^{(k_{1})}_{i, j, \dots, j}\\
       c^{(k_{1}+1)}_{i, i, i, \dots, i}  & c^{(k_{1}+1)}_{i, i, i, \dots, j} & \dots & c^{(k_{1}+1)}_{i, i, j, \dots, j}\\
       c^{(k_{1}+1)}_{j, i, i, \dots, i}  & c^{(k_{1}+1)}_{j, i, i, \dots, j} & \dots & c^{(k_{1}+1)}_{j, i, j, \dots, j}\\
       \vdots & \vdots & \ddots & \vdots \\
       c^{(k_{2})}_{i, \dots, i, i, i, \dots, i, i}  & c^{(k_{2})}_{i, \dots, i, i, i, \dots, i, j} & \dots & c^{(k_{2})}_{i, \dots, i, i, j, \dots, j, j}\\
       \vdots & \vdots & \ddots & \vdots \\
       c^{(k_{2})}_{j, \dots, j, i, i, \dots, i, i}  & c^{(k_{2})}_{j, \dots, j, i, i, \dots, i, j} & \dots & c^{(k_{2})}_{j, \dots, j, i, j, \dots, j, j}
    \end{array}
    \right].
\end{align}
Define $A^{(k_{1}, k_{2})}_{i,j}$ analogously by swapping $i$ and $j$. 
Proposition~\ref{thm: latent number} allows us to identify $\ell$ in the model \eqref{eq: two variables} and the causal order between $X_i$ and $X_j$.
\begin{proposition}[\cite{schkoda2024causal}]
\label{thm: latent number}
    Assume that $X_{i}$ and $X_{j}$ are two observed variables where $X_{j} \notin \mathrm{Anc}(X_{i})$. Let $d := \min(\sum^{k_{2}-k_{1}+1}_{i=1}i, k_{1})$.
    Then, 
    
    1.  $A^{(k_{1}, k_{2})}_{{j}, {i}}$ generically has rank $\min(\ell+1, d)$.
    
    2.  If $m^{O}_{ji}\!\ne\!0$, $A^{(k_{1}, k_{2})}_{{i}, {j}}$ generically has rank $\min(\ell+2,\!d)$.
    
    3.  If $m^{O}_{ji}\!=\!0$, $A^{(k_{1}, k_{2})}_{{i}, {j}}$ generically has rank $\min(\ell+1,\!d)$.
\end{proposition} 
According to \cite{schkoda2024causal}, the smallest possible choice of $(k_1,k_2)$ is 
$(\ell+2, (\ell + 2) + \lceil (-3 + \sqrt{8\ell + 17})/2 \rceil)$.
Define ${A}^{(\ell)}_{j,i}$ as $A^{(k_1, k_2)}_{j,i}$, where $(k_1, k_2)$ is this choice.

Proposition~\ref{thm: latent number} provides a practical criterion for determining $\ell$. If the true number of confounders is $\ell$, then Proposition~\ref{thm: latent number} implies that the matrix $A_{j,i}^{(\ell)}$ generically has rank $\ell+1$. Consequently, for any $r<\ell$,
$\operatorname{rank}\left(A_{j,i}^{(r)}\right)>r+1$. 
Therefore, $\ell$ is the smallest value of $r$ satisfying $\operatorname{rank}\left( A_{j,i}^{(r)} \right) = r+1$. 
Moreover, Items 2 and 3 of Proposition~\ref{thm: latent number} determine whether an ancestral relationship exists between $X_i$ and $X_j$, and, if so, identify its direction.

Let $\tilde{A}^{(\ell)}_{{j}, {i}}$ be a matrix obtained by adding $(1,\!m,\!\dots,\!m^{\ell+1})$ as the first row of $A^{(\ell)}_{{j}, {i}}$.
\begin{proposition}[\cite{schkoda2024causal}]
\label{thm: estimate b}
    Consider the determinant of an $(\ell+2)\times (\ell+2)$ minor of $\tilde{A}^{(\ell)}_{{j}, {i}}$ that contains the first row and treat it as a polynomial in $m$. 
    Then, the roots of this polynomial are $m^{O}_{ji}, m^{OL'}_{j1}, \cdots, m^{OL'}_{j\ell}$.
\end{proposition}
Proposition~\ref{thm: estimate b} identifies total effects 
$m^{O}_{ji}, m^{OL'}_{j1}, \cdots, m^{OL'}_{j\ell}$ up to permutation. 
\begin{proposition}[\cite{schkoda2024causal}]
\label{lem: estimate e cumulant}
Define
$\bm{m}_{ji}^{O}:=
\left(
1,
m_{ji}^{O},
(m_{ji}^{O})^{2},
\ldots,
(m_{ji}^{O})^{k-1}
\right)^{\top}$. 
Similarly, define
\(
\bm{m}_{j1}^{OL'},\ldots,\bm{m}_{j\ell}^{OL'}
\).
Then, the system of equations
\begin{align}
    \label{eq: system}
    &\left[
        \bm{m}_{ji}^O, \bm{m}_{j1}^{OL}, \dots, \bm{m}_{j\ell}^{OL}
    \right]
    \left[\kappa^{(k)}({v}_{i}), \kappa^{(k)}(L^\prime_{1}), \dots, \kappa^{(k)}(L^\prime_{\ell}) 
    \right]^{\top}\notag\\
    &\qquad =\!\left[
        c^{(k)}_{i, i, \dots, i}, c^{(k)}_{i, i, \dots, j}, \dots, c^{(k)}_{i, j, \dots, j}
    \right]^{\top}\!,
\end{align}
is generically uniquely solvable if $k \geq \ell+1$.
\end{proposition}

When \eqref{eq: system} is uniquely solvable, each root returned by
Proposition~\ref{thm: estimate b} is associated with a unique
$k$-th order cumulant of a disturbance or latent variable, 
\[
\left(
m_{ji}^O, \kappa^{(k)}(v_i)
\right), 
\left(m_{ji}^{OL}, \kappa^{(k)}(L_1^\prime)
\right), \ldots,
\left(
m_{j \ell}^{OL}, \kappa^{(k)}(L_\ell^\prime)
\right). 
\]
Although the roots and cumulants themselves are
identified only up to permutation, the correspondence between a root
and its associated cumulant is uniquely determined.
Hereafter, without any additional explanation, we focus on orders $k$ for which \eqref{eq: system} is solvable.

\section{Proposed Method}
\label{sec: proposed method}
We propose a method to recover the sparsest DAG under the genericity assumption. In the following, let $\mathcal{G}$ denote the sparsest DAG in the observational equivalence class. 
The proposed method first applies Proposition~\ref{thm: latent number} to infer all ancestral relationships among the observed variables and identify those with no observed ancestors as observed sources. The same proposition also estimates the number of latent confounders for each pair of observed variables. We then use the total effects obtained from Proposition~\ref{thm: estimate b}, together with their correspondences to the $k$-th order cumulants of the disturbances and latent confounders established by Proposition~\ref{lem: estimate e cumulant}, to determine the parent--child relationships between each observed source and its descendants. We then identify the sparsest latent-to-observed structure by selecting, among the candidate total effects obtained above, those that minimize the numbers of latent confounders between pairs of observed variables. The identified sources are then residualized from the remaining observed variables, after which the above estimation procedure is repeated on the updated variables. Recursively repeating this process recovers the sparsest DAG. Proofs of all theorems in this section are given in Appendix~\ref{sec: proof of thm}. 

Based on the pairwise ancestral relationships estimated using Proposition~\ref{thm: latent number}, variables without observed ancestors are treated as observed sources. 
Let $X_s$ be one such source.
Let $\widetilde{\mathrm{Ch}}(X_s)\subset\mathrm{Ch}(X_s)$ be the set of observed variables whose observed ancestors consist only of observed sources and include $X_s$.
To identify the remaining children of $X_s$, we recursively examine its descendants while maintaining two sets: the closed set $\bm X_{\mathrm{closed}}$, containing variables already identified as children of $X_s$, and the open set $\bm X_{\mathrm{open}}$, containing descendants whose parent-child relationships with $X_s$ have yet to be determined. Initially, 
\begin{align}
    \label{eq: open list}
    &\bm{X}_{\mathrm{\mathrm{closed}}} = \widetilde{\mathrm{Ch}}(X_{s}), \notag
    \\
    &{\bm{X}}_{\mathrm{\mathrm{open}}}
    = \Big\{ X_{j}\in \mathrm{Des}(X_{s})\setminus X_{\mathrm{\mathrm{closed}}}:{ \forall X_{i}\in \bm{X}_{{\mathrm{\mathrm{closed}}}},~\mathrm{Anc}(X_{j}) \cap \mathrm{Des}(X_{i}) \subseteq \bm{X}_{\mathrm{\mathrm{closed}}}} \Big\}.
\end{align}
By definition of $\bm{X}_{\mathrm{\mathrm{open}}}$, no observed variable in $\bm{X} \setminus \bm{X}_{\mathrm{\mathrm{closed}}}$ lies on a directed path between $\bm{X}_{\mathrm{\mathrm{closed}}}$ and $\bm{X}_{\mathrm{\mathrm{open}}}$.
\begin{lemma}
\label{thm: total effects and edges}
    For any $X_{j} \in \bm{X}_{\mathrm{open}}$,
    $b_{js}$ can be written as
    \begin{align}
    \label{eq: total effects and edges}
        b_{js}
        = 
        m^{O}_{js}
        - \sum_{i: X_{i} \in \widetilde{\mathrm{Ch}}(X_{s})} b_{is} \, m^{O}_{ji}.
    \end{align}
\end{lemma}
According to Lemma~\ref{thm: total effects and edges}, the parent-child relationship between $X_s$ and $X_j$ can be determined by testing whether $b_{js}=0$. 
However, Proposition~\ref{thm: estimate b} identifies the total effects appearing in \eqref{eq: total effects and edges} only up to permutation, so $b_{js}$ cannot be computed directly.

Denote by $\mathcal I$ and $\mathcal{J}$ the index sets of $\bm{X}_{\mathrm{closed}}$ and $\bm{X}_{\mathrm{open}}$, respectively. 
For $j\in\mathcal J$ and $i\in\mathcal I$, let $\bm m^O_{js}$ and $\bm b_{is}$ denote the sets of candidate values of $m^O_{js}$ and $b_{is}$, respectively, returned by Proposition~\ref{thm: estimate b}. 
Initially, $b_{is}=m^O_{is}$.

As discussed in the previous section, Proposition~\ref{lem: estimate e cumulant} associates each candidate total effect with the corresponding $k$-th order cumulant of a disturbance or latent confounder. For any $i\in\mathcal I$ and $j\in\mathcal J$, the systems \eqref{eq: system} for the pairs $(X_s,X_i)$ and $(X_s,X_j)$ always share the $k$-th order cumulant of $e_s$. Under the genericity assumption, distinct disturbances and latent confounders have distinct $k$-th order cumulants. Therefore, we can choose $\alpha_{js}\in\bm m^O_{js}$ and $\beta_{is}\in\bm b_{is}$ so that they correspond to the same $k$-th order cumulant. Such a choice is not necessarily unique, since the two systems may also share the $k$-th order cumulants of latent confounders common to the pairs $(X_s,X_i)$ and $(X_s,X_j)$.

Define the residualized variables
\[
\tilde{X}_i=X_i-\beta_{is}X_s,\qquad \tilde{X}_j=X_j-\alpha_{js}X_s.
\]
By Lemma~\ref{lem: induced} in Appendix~\ref{sec: some thm}, $\tilde{X}_i$ and $\tilde{X}_j$ can be regarded as observed variables in an induced canonical LvLiNGAM. Applying Propositions~\ref{thm: estimate b} and \ref{lem: estimate e cumulant} to this induced model yields a candidate set for the total effect from $\tilde{X}_i$ to $\tilde{X}_j$. We denote this set by $\bm m^O_{ji.s}$. 

Theorem \ref{cor: total effects and edges} provides a criterion for determining whether $X_s \in \mathrm{Pa}(X_j)$ for each $j \in \mathcal{J}$.

\begin{theorem}
\label{cor: total effects and edges}
Under the genericity assumption, $X_s\notin \mathrm{Pa}(X_j)$ if and only if there exist
\begin{align}
\label{eq: total effects and edges generalized}
&\alpha_{js}\in \bm{m}^{O}_{js},\quad
\bm{\beta}_{\mathcal{I},s}\in \prod_{i\in\mathcal{I}}\bm{b}_{is},\quad
\bm{\gamma}_{j,\mathcal{I}.s} \in \prod_{i\in\mathcal{I}}\bm{m}^{O}_{ji.s} \notag\\
&\qquad \text{s.t.}\qquad
\alpha_{js}-\bm{\beta}_{\mathcal{I},s}^{\top}\bm{\gamma}_{j,\mathcal{I}.s}=0,
\end{align}
where $\alpha_{js}$ and $\beta_{is}$, $i \in \mathcal{I}$ are chosen so that their associated $k$-th order cumulants are equal.
\end{theorem}
Once the parent-child relationship between $X_{s}$ and $X_{j}$ is determined using Theorem~\ref{cor: total effects and edges}, we update $\bm{X}_{\mathrm{closed}}$, $\bm{X}_{\mathrm{open}}$, and $\widetilde{\mathrm{Ch}}(X_{s})$ as follows:
{\small
\begin{align*}
    \bm{X}_{\mathrm{closed}} \gets \bm{X}_{\mathrm{closed}} \cup \{X_{j}\}, 
    \ \bm{X}_{\mathrm{open}} \gets \bm{X}_{\mathrm{open}} \setminus \{X_{j}\}, 
    \
    \widetilde{\mathrm{Ch}}(X_{s}) \gets 
    \begin{cases}
       \widetilde{\mathrm{Ch}}(X_{s}),  & X_{s} \not \in \mathrm{Pa}(X_{j}), \\
       \widetilde{\mathrm{Ch}}(X_{s}) \cup \{X_{j}\},  & X_{s} \in \mathrm{Pa}(X_{j}).
    \end{cases}
\end{align*}
}
When $\bm{X}_{\mathrm{open}}$ becomes empty, it is reinitialized using
\eqref{eq: open list}.
As shown in the proof of Lemma~\ref{thm: total effects and edges} in Appendix~\ref{sec: proof of thm}, Lemma~\ref{thm: total effects and edges} still holds after updating $\widetilde{\mathrm{Ch}}(X_s)$. 
The coefficients of the directed edges from $X_s$ to the variables in $\widetilde{\mathrm{Ch}}(X_s)$ can also be computed by Lemma~\ref{thm: total effects and edges}.
Therefore, the parent-child relationships of $X_s$ and variables in $\bm{X}_\mathrm{open}$ can still be identified by Theorem~\ref{cor: total effects and edges}.

After determining all parent-child relationships between $X_s$ and its descendants, we residualize each descendant $X_j\in\mathrm{Des}(X_s)$ with respect to $X_s$ before proceeding to the next iteration:
\begin{align}
\label{eq: update}
\tilde{X}_j = X_j-\alpha_{js}X_s.
\end{align}
For each $j\in\mathcal J$, choose $\alpha_{js}$ from triples
$(\alpha_{js},\bm{\beta}_{\mathcal I,s},\bm{\gamma}_{j,\mathcal I,s})$
satisfying \eqref{eq: total effects and edges generalized} whenever available; and otherwise, choose it randomly from $\bm m^O_{js}$, so that all selected values of $\alpha_{js}$ correspond to the same $k$-th order cumulant.

Let $\mathcal R=[p]\setminus\{s\}$. For each $j\in\mathcal R$, let $\alpha_{js}$ be defined as above when $X_j\in\mathrm{Des}(X_s)$, and set $\alpha_{js}=0$ when $X_j\notin\mathrm{Des}(X_s)$. Let
\begin{align*}
    \bm X_{\mathcal R}=(X_j)_{j\in\mathcal R}^{\top}, \qquad
    \bm\alpha_{\mathcal R,s}=(\alpha_{js})^\top_{j\in\mathcal R}.
\end{align*}
Then, the residualized variables $\tilde{\bm X}_{\mathcal R}$ after removing $X_s$ are given by
\begin{equation}
\label{eq: update_vec}
\tilde{\bm X}_{\mathcal R} = (\tilde{X}_j)_{j\in\mathcal R}^{\top}
:=
\bm X_{\mathcal R}
-
\bm\alpha_{\mathcal R,s}X_s. 
\end{equation}
By Lemma~\ref{lem: induced} and Corollary~\ref{cor: invariant B} in Appendix~\ref{sec: some thm}, we obtain the following theorem. 
\begin{theorem}
\label{thm: unchanged}
Under the genericity assumption, $\tilde{\bm X}_{\mathcal R}$ admits an LvLiNGAM representation
whose observed DAG is the induced subgraph of $\mathcal G^O$ on $\bm X\setminus\{X_s\}$.
\end{theorem}

By Theorem~\ref{thm: unchanged}, after the update \eqref{eq: update},
$\tilde{\bm X}_{\mathcal R}$ admits an LvLiNGAM representation whose observed DAG coincides with the induced subgraph of $\mathcal G^O$ on $\bm X\setminus\{X_s\}$.
Although the above selection procedure may yield different vectors
$\bm{\alpha}_{\mathcal R,s}$ corresponding to different $k$-th order cumulants, all such choices induce the same observed DAG.

Let $\mathcal A_s$ denote the set of all vectors
$\bm{\alpha}_{\mathcal R,s}$ obtained by selecting different common $k$-th order cumulants in the above procedure. 
Although every vector in $\mathcal A_s$ induces the same observed DAG over the remaining variables, different choices may yield different latent-to-observed structures.
Since $\mathcal G$ is assumed to be the sparsest DAG in its observational equivalence class, the remaining task is to identify the vector in $\mathcal A_s$ that yields the sparsest latent-to-observed structure.
Specifically, for each
$\bm{\alpha}_{\mathcal R,s}\in\mathcal A_s$, define
\[
L(\bm{\alpha}_{\mathcal R,s})
=
\sum_{i,j\in\mathcal R}
\ell_{ij}(\bm{\alpha}_{\mathcal R,s}),
\]
where $\ell_{ij}(\bm{\alpha}_{\mathcal R,s})$ denotes the number of latent confounders between the residualized variables $\tilde X_i$ and $\tilde X_j$ identified by Proposition~\ref{thm: latent number}.
Since every vector in $\mathcal A_s$ induces the same observed DAG and the same number of latent variables, minimizing $L(\bm{\alpha}_{\mathcal R,s})$ is equivalent to selecting the sparsest latent-to-observed structure, and hence the sparsest DAG.
Accordingly, we select any minimizer
\begin{equation}
\label{eq: alpha star}
\bm{\alpha}_{\mathcal R,s}^{*}
\in
\operatorname*{arg\,min}_{\bm{\alpha}_{\mathcal R,s}\in\mathcal A_s}
L(\bm{\alpha}_{\mathcal R,s}).
\end{equation}
\begin{theorem}
\label{thm: unchanged OL}
Under the genericity assumption, the sparsest latent-to-observed structure over \(\tilde{\bm X}_{\mathcal R}\) obtained from \(\bm{\alpha}_{\mathcal R,s}^{*}\) coincides with the structure obtained from the induced subgraph of $\mathcal G^{OL}$ on $\bm{V} \setminus \{X_{s}\}$ by absorbing every latent variable having only one observed child into the disturbance of that child.
\end{theorem}
Thus, the update \eqref{eq: update} becomes
\begin{align}
\label{eq: update *}
    \tilde{X}_{j} = X_{j} - \alpha^{*}_{js}X_{s}, \quad \forall X_{j} \in \bm{X}\setminus \{X_{s}\},
\end{align}
where $\bm{\alpha}^{*}_{\mathcal R,s}=(\alpha^{*}_{js})_{j\in\mathcal R}$ is selected by \eqref{eq: alpha star}.
If multiple such choices of $\bm{\alpha}^*_{\mathcal R,s}$ exist, one is selected at random.

At this iteration, the selected aligned group gives the total-effect column associated with the removed source in the selected sparsest representation.
The remaining aligned groups of candidate total effects, whose entries correspond to the same cumulant, are recorded as total-effect columns of latent variables in the mixing matrix.

By Theorems~\ref{thm: unchanged} and \ref{thm: unchanged OL}, updating the variables using
$\bm{\alpha}_{\mathcal R,s}^{*}$ yields an LvLiNGAM whose DAG coincides with the induced subgraph of $\mathcal G$ induced by removing $X_s$, where every latent variable having only one observed child is absorbed into the disturbance of that child. 
Applying this procedure recursively therefore recovers the entire sparsest DAG $\mathcal G$. 
This result is formalized in the following theorem.
\begin{theorem}
\label{thm: identify mixing matrix}
Under the genericity assumption, the proposed top-down procedure identifies all total-effect columns of the mixing matrix up to permutation consistent with the sparsest DAG over $\bm{V}$.
\end{theorem}
According to the selected combination of total effects in the update~\eqref{eq: update *}, the proposed method obtains the mixing matrix among the observed variables, namely $(\bm I-\bm B)^{-1}$, corresponding to the sparsest DAG over $\bm X$. 
Hence, $\bm B$ can be estimated from $(\bm I-\bm B)^{-1}$.
In finite samples, $\bm{B}$ can be pruned by enforcing consistency with the estimated parent-child relationships.
Moreover, the remaining total-effect columns corresponding to latent sources form $(\bm I-\bm B)^{-1}\bm{\Lambda}$, which ensures the sparsest latent-to-observed bipartite graph between $\bm{L}$ and $\bm{X}$.
Multiplying $(\bm{I}-\bm{B})$ by $(\bm{I}-\bm{B})^{-1}\bm{\Lambda}$ yields an estimate of $\bm{\Lambda}$, and thus identifies the directed edges from $\bm{L}$ to $\bm{X}$.
In finite samples, $\bm{\Lambda}$ can be pruned 
by setting entries whose absolute values are below a predefined threshold to zero.

The update \eqref{eq: update *} residualizes the descendants by removing the effects of the observed source. Higher-order cumulants are then recomputed from the residualized variables, and the procedure based on Propositions~\ref{thm: latent number}--\ref{thm: estimate b} is recursively applied to the residualized variables. 

In practice, higher-order cumulants are estimated from finite samples, and their estimation accuracy generally deteriorates as the order increases. ReLVLiNGAM updates the cumulants of descendant variables after subtracting the contributions of the disturbance and latent confounders associated with the observed source. Since this update explicitly relies on the estimated higher-order cumulants of these latent variables and disturbances, errors in those estimates may directly affect the updated cumulants used in subsequent total-effect estimation.


The proposed method also uses higher-order cumulants, but only to match candidate total effects across different variable pairs. Once the matching is completed, the update is performed by residualizing the observed variables themselves rather than by updating cumulants. Consequently, the update does not explicitly rely on the estimated higher-order cumulants of individual disturbances or latent confounders, which is expected to reduce the impact of their estimation errors on downstream inference. 

In addition, unlike ReLVLiNGAM, the proposed method does not rely on low-order cumulants to recursively estimate disturbance cumulants. 
Instead, it only requires an order $k$ for which \eqref{eq: system} is uniquely solvable. 
Specifically, the value of $k$ can be determined by increasing it from $k=2$ and choosing the smallest one for which the corresponding linear system in \eqref{eq: system} has full column rank.
Hence, the proposed method does not suffer from the local restriction imposed by ReLVLiNGAM.

The proposed procedure is summarized in Algorithm~\ref{alg: proposed}.
We also provide a detailed example illustrating the proposed algorithm on the models in Figure~\ref{fig: equivalence class of IV graph} in Appendix~\ref{sec: example}.

\begin{algorithm}[!t]
\DontPrintSemicolon
\SetAlgoInsideSkip{smallskip}
\SetAlCapSkip{0.3em}
\SetInd{0.25em}{0.55em}
{
\footnotesize
\hrule
\BlankLine
\caption{Proposed Method}
\label{alg: proposed}

\BlankLine
\hrule
\BlankLine

\Input{Observed data matrix $\bm{X}\in\mathbb{R}^{n\times p}$}
\Output{Estimated causal graph $\widehat{\mathcal{G}}$}

\textbf{Initialization}\;
$\widehat{\mathcal{G}}^{O}\gets (\bm{X},\emptyset)$,\;
$\widehat{\bm{M}}\gets \bm{I}_{p\times p}$,\;
$\widehat{\bm{M}}_{\mathrm{latent}}\gets \emptyset$\;

Identify ancestral relationships among $\bm{X}$ and observed sources $\bm{X}_{s}$ by Proposition~\ref{thm: latent number}\;

\While{$\bm{X}_{s}\neq \emptyset$}{
    $\bm{X}_{s,\mathrm{next}}\gets \emptyset$\;

    \ForEach{$X_{s}\in \bm{X}_{s}$}{
        Identify $\widetilde{\mathrm{Ch}}(X_{s})$, $\bm{X}_{{\mathrm{open}}}$, and $\bm{X}_{{\mathrm{closed}}}$\;

        Compute all possible total effects from $X_s$ and the corresponding disturbance cumulants by Propositions~\ref{thm: estimate b} and~\ref{lem: estimate e cumulant} 
        \;


        \While{$\bm{X}_{{\mathrm{open}}}\neq \emptyset$}{
            Compute total effects from $\widetilde{\mathrm{Ch}}(X_s)$ to $\bm{X}_{{\mathrm{open}}}$ after regressing out $X_s$ under each possible total effect of $X_s$ by Proposition~\ref{thm: estimate b}\;

            Determine the parent--child relationships between $X_s$ and the nodes in $\bm{X}_{\mathrm{\mathrm{open}}}$ by Theorem~\ref{cor: total effects and edges}\;

            Update $\widetilde{\mathrm{Ch}}(X_s)$, $\bm{X}_{{\mathrm{open}}}$, and $\bm{X}_{{\mathrm{closed}}}$\;
        }

        Add all identified edges $X_s\rightarrow X_j$ for $X_j\in \widetilde{\mathrm{Ch}}(X_s)$ into $\widehat{\mathcal{G}}^{O}$\;

        
        Remove the effects of $X_s$ from $\mathrm{Des}(X_s)$ using the update~\eqref{eq: update *} with the selected total effects\;



        Replace the corresponding entries in $\widehat{\bm{M}}$ with all selected total effects in \eqref{eq: update *}\;

        Append other possible total effects not used in Equation~\eqref{eq: update *} to $\widehat{\bm{M}}_{\mathrm{latent}}$\;

        Identify newly emerging observed sources according to the ancestral relationships by Proposition~\ref{thm: latent number} and add them to $\bm{X}_{s,\mathrm{next}}$\;
    }

    $\bm{X}_{s}\leftarrow \bm{X}_{s,\mathrm{next}}$\;
}

$\widehat{\bm{B}} = \bm{I} - \widehat{\bm{M}}^{-1}$, $\widehat{\bm{\Lambda}} = \widehat{\bm{M}}^{-1}\widehat{\bm{M}}_{\mathrm{latent}}$, and construct $\widehat{\mathcal{G}}^{OL}$ from $\widehat{\Lambda}$\;

\Return{$\widehat{\mathcal{G}} = \widehat{\mathcal{G}}^{O}\cup \widehat{\mathcal{G}}^{OL}$}\;}
\BlankLine
\hrule
\end{algorithm}

\section{Simulations}
\label{sec: simulation}
\begin{figure}[h]
    \centering
    \subfigure[\centering case I]{
    \includegraphics[height = 0.149\textwidth]{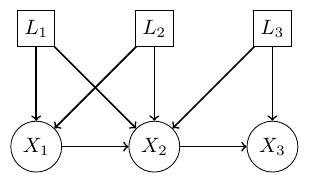}
    }
    \subfigure[\centering case II]{
    \includegraphics[height = 0.149\textwidth]{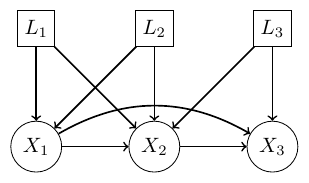}
    }
    \subfigure[\centering case III]{
    \includegraphics[height = 0.149\textwidth]{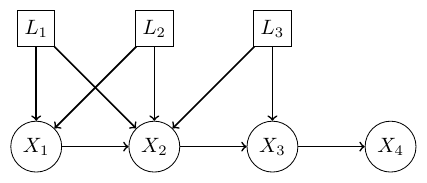}
    }
    \caption{Three models used for simulations.}
    \label{fig: simulation}
\end{figure}
This section reports simulation results\footnote{Code is available at \url{https://anonymous.4open.science/r/Test_20260701}} on the three causal DAGs in Figure~\ref{fig: simulation}, which are not handled by the original ReLVLiNGAM. 
We compare the proposed method with ReLVLiNGAM \citep{schkoda2024causal} to examine whether it overcomes ReLVLiNGAM’s local restriction.
To evaluate the proposed method under an oracle setting, we also report results for a variant of the proposed method that is provided with the true ancestral relationships and the true number of latent confounders. 
The local restriction in the original ReLVLiNGAM arises because cumulant updates rely on low-order cumulants. When $k < \ell + 1$, the system \eqref{eq: system} becomes underdetermined and is no longer uniquely solvable. The local restriction is imposed to avoid this situation. 
This restriction can be removed by simply using sufficiently high-order cumulants when solving \eqref{eq: system}. For completeness, Appendix~\ref{sec: simple improvement} describes a simple modification of ReLVLiNGAM that replaces low-order cumulants with sufficiently high-order ones when solving \eqref{eq: system}. We also include this modified version in the experimental comparison.

\subsection{Settings}
\label{sec: settings}

All disturbances and latent variables are sampled from $\mathrm{Lognormal}(-1.1,0.8)$, and then centered to have zero mean. 
To avoid identical distributions among the disturbances and latent variables, consistent with the genericity assumption, each variable is further multiplied by an independent scale sampled from $\mathrm{Uniform}(0.9,1.1)$.
For each latent variable $L_i$, the coefficient from $L_i$ to its observed child with the highest causal order is fixed to one. 
All other coefficients in $\bm{\Lambda}$ and $\bm{B}$ are independently drawn from $\mathrm{Uniform}(0.5, 0.8)$.
Since all causal coefficients are positive, the effects along different paths cannot cancel each other out. 
Moreover, the selected distributions have nonzero higher-order cumulants.
The sample size $N$ is set to 1K, 10K, 100K, and 1M, and each experiment is repeated $50$ times.
We evaluate the performance of the methods using the following metrics: 
\begin{itemize}[topsep=0pt]
\setlength{\parskip}{0cm}
\setlength{\itemsep}{0cm}
    \item $\mathrm{N}_{\mathrm{tol}}$ and $\mathrm{N}_{\mathrm{obs}}$: the number of runs in which the DAGs of $\mathcal{G}$ and $\mathcal{G}^O$ are correctly recovered, respectively (see Figure~\ref{fig: result N});
    \item $\mathrm{PRE}$, $\mathrm{REC}$, and $\mathrm{F1}$: the average precision, recall, and F1-score of the estimated edges of $\mathcal{G}$ (see Figure~\ref{fig: result pre}).
\end{itemize}

In the proposed method, we consider \eqref{eq: total effects and edges generalized} in 
Theorem~\ref{cor: total effects and edges} to hold in finite-sample settings if $\vert\alpha_{js}-\bm{\beta}_{\mathcal{I},s}^{\top}\bm{\gamma}_{j,\mathcal{I}.s}\vert < \tau_0 \in \{0.2,0.15,0.125,0.1\}$ for increasing sample sizes.
We further prune an edge from $\bm{L}$ to $\bm{X}$ if its absolute coefficient is below $\tau_{0, L}=0.3$. 
For the other methods, we also apply $\tau_{0}$ and $\tau_{0 ,L}$ to the estimated coefficient matrices to prune edges.
For the original and modified ReLVLiNGAM methods, we enumerate all candidate DAGs from the estimated mixing matrices, prune edges in the estimated coefficient matrices using $\tau_{0}$ and $\tau_{0,L}$, and select the sparsest DAG. 

Although both the proposed method and ReLVLiNGAM employ Proposition~\ref{thm: latent number}, they differ in how the rank of $A^{(k_{1}, k_{2})}_{{j}, {i}}$ is determined.
Let $\sigma_{r}$ be the $r$-th largest singular value of $A^{(k_{1}, k_{2})}_{{j}, {i}}$ and let $\tau_{s}$ and $\tau_{cs}$ be two predefined thresholds.
ReLVLiNGAM treats $\sigma_{r}$ as zero if $\sigma_{r}/\sigma_{1} \le \tau_{s}$.
In contrast, the proposed method sets $\sigma_{r}=0$ if $1-\sum_{i \in [r]} \sigma_{i}/ \sum_{i \in [d]} \sigma_{i} \le \tau_{cs}$, where $d$ is the number of singular values of $A^{(k_{1}, k_{2})}_{{j}, {i}}$ and is defined as in Proposition \ref{thm: latent number}.
Following \citet{schkoda2024causal}, we also impose an upper bound $\ell_{\mathrm{highest}}$ on the number of latent variables, with $\ell_{\mathrm{highest}}=2$ in cases~I and III, and $\ell_{\mathrm{highest}}=4$ in case~II. 
We set $\tau_s=0.008(i-1)/N^{0.125}$ for the original ReLVLiNGAM and $\tau_{cs}=0.002+0.0005(i-1)$ for our method, where $i$ denotes the depth from the observed source to reflect increasing estimation error with depth.
The settings of the modified ReLVLiNGAM follow those of the proposed method.
To estimate the cumulants of the disturbances and latent variables, both the proposed method and the modified ReLVLiNGAM increase $k$ from $k=2$ until the system in \eqref{eq: system} is uniquely solvable, and then use the resulting value of $k$ in Proposition~\ref{lem: estimate e cumulant}.
\begin{figure}[t]
    \centering
    \includegraphics[width = 0.95\textwidth]{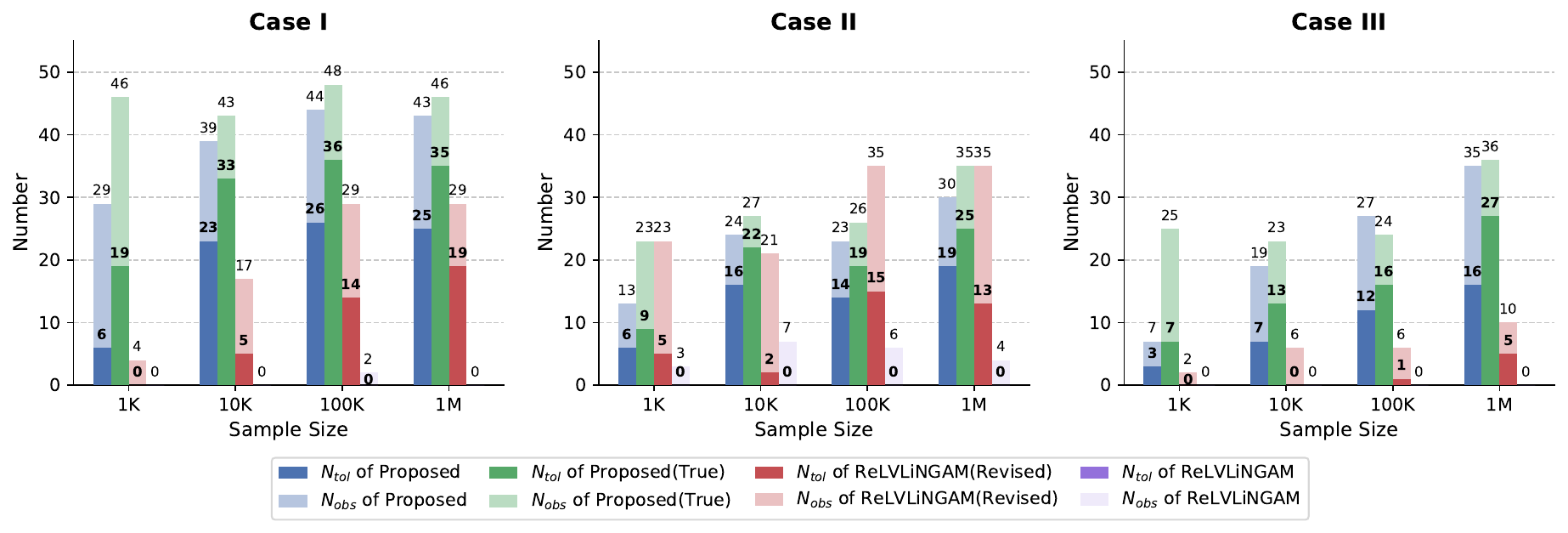}
    \caption{The performances in $\mathrm{N}_{\mathrm{tol}}$ and $\mathrm{N}_{\mathrm{obs}}$.}
    \label{fig: result N}
\end{figure}
\begin{figure}[t]
    \centering
    \includegraphics[width = 0.95\textwidth]{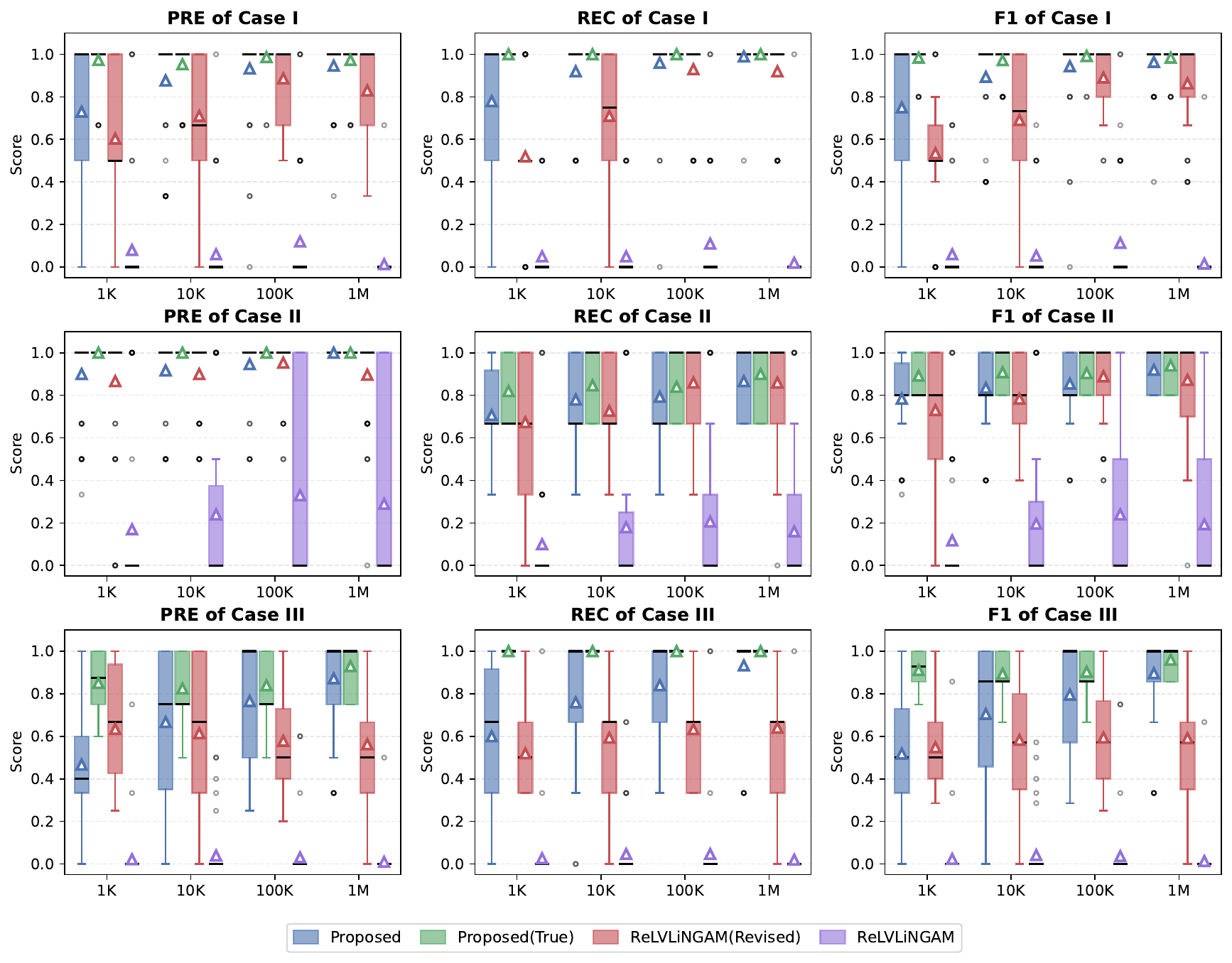}
    \caption{The performances of PRE, REC, and F1-score.}
    \label{fig: result pre}
\end{figure}
\subsection{Discussion}
\label{sec: discussion}
Figure~\ref{fig: result N} shows that the proposed method outperforms the original ReLVLiNGAM on all evaluation metrics across all settings. Since the DAGs in Figure~\ref{fig: simulation} do not satisfy the local restriction, the original ReLVLiNGAM fails to recover the correct DAG even as the sample size increases. In contrast, the estimation accuracy of the proposed method improves steadily with increasing sample size.

For reference, we also report the performance of the oracle version of the proposed method. Compared with the standard version, the oracle version achieves substantially higher accuracy when the DAG is sparse or the sample size is small. This result suggests that the accuracy of estimating ancestral relationships and the number of pairwise confounders has a non-negligible impact on the accuracy of DAG recovery.

The boxplots of PRE, REC, and F1 score in Figure~\ref{fig: result pre} show that, for the proposed method, both the mean (triangles) and the median (horizontal bars) of these metrics increase as $N$ increases,
indicating increasingly accurate recovery of parent--child relationships over $\bm{X}$.

The modified ReLVLiNGAM achieves performance comparable to that of the proposed method when the sample size is large in case II, but performs worse in sparse settings (cases I and III), likely because it lacks an effective edge-pruning strategy from finite samples. 
It is generally inferior to the proposed method in recovering the $\bm{L}\!\to\!\bm{X}$ edges, except in case II with 100K samples. 
This may be because errors in estimating higher-order disturbance cumulants can propagate to subsequent cumulant computations.
In contrast, the proposed method uses estimated disturbance cumulants only to match candidate total effects with the same associated cumulant. 
Since the cumulants are not recursively updated and propagated to subsequent iterations, the effect of cumulant estimation errors is expected to be less severe. 


\section{Real Data}
\label{sec: real-world}
We further evaluate the practical usefulness of the proposed method by applying it, together with ParceLiNGAM \citep{tashiro2014parcelingam}, RCD \citep{Maeda2020}, and the original and modified versions of ReLVLiNGAM \citep{schkoda2024causal}, to the Sachs protein dataset from \cite{Sachs2005}.
\begin{figure*}[t]
    \centering
    \subfigure[\mbox{Reference DAG}]{
        \includegraphics[width=0.22\textwidth, page=1]{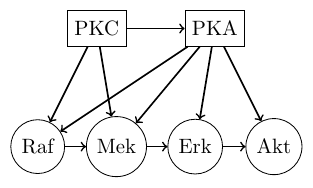}
    }
    \subfigure[\centering Canonical model]{
        \includegraphics[width=0.22\textwidth, page=2]{sachs_result_1.pdf}
    }
    \subfigure[\centering ParceLiNGAM]{
    \includegraphics[width=0.22\textwidth, page=3]{sachs_result_1.pdf}
    }
    \subfigure[\centering RCD]{
    \includegraphics[width=0.22\textwidth, page=4]{sachs_result_1.pdf}
    }
    \subfigure[\centering ReLL]{
    \includegraphics[width=0.22\textwidth, page=5]{sachs_result_1.pdf}
    }
    \subfigure[\mbox{modified ReLL}]{
    \includegraphics[width=0.22\textwidth, page=6]{sachs_result_1.pdf}
    }
    \subfigure[\centering Proposed method]{
    \includegraphics[width=0.22\textwidth, page=7]{sachs_result_1.pdf}
    }
    \caption{The results of different methods applied to the Sachs dataset.}
    \label{fig: application}
\end{figure*}
The Sachs dataset \citep{Sachs2005} records expression levels of phosphorylated proteins and phospholipids in human immune cells and contains 11 variables and 7,467 samples. The dataset is accompanied by a reference signaling network constructed from biological knowledge. Based on this reference network, we treat PKC and PKA as latent confounders and use Raf, Mek, Erk, and Akt as observed variables. The DAG for the model is shown in Figure~\ref{fig: application}(a).
Although the DAG in Figure~\ref{fig: application}(a) is not a canonical LvLiNGAM, it can be transformed into its canonical form in Figure~\ref{fig: application}(b) without changing the observed structure by applying Algorithm~A of \cite{hoyer2008estimation}.

The thresholds $\tau_s$ and $\tau_{cs}$ for the proposed method and ReLVLiNGAMs are set in the same manner as in the numerical experiments in Section~\ref{sec: simulation}. 
The upper bound $\ell_{\mathrm{highest}}$ is set to $2$ for both methods.
Here, instead of employing a hard threshold, we verify~\eqref{eq: total effects and edges generalized} in Theorem~\ref{cor: total effects and edges} using 99\% bootstrap confidence intervals constructed from 400 bootstrap resamples.
ParceLiNGAM and RCD use the Hilbert--Schmidt independence criterion (HSIC; \citealp{Gretton2007}) to infer causal directions among observed variables. 
RCD sets the HSIC significance level to 0.01. 
ParceLiNGAM additionally applies Fisher’s method to combine HSIC p-values and uses a significance level of 0.1 for Fisher’s test. 
RCD also uses the Pearson test and the Shapiro-Wilk test, both at the 0.01 significance level.
For the original ReLVLiNGAM, since it outputs a very dense $\bm{L}\!\to\!\bm{X}$ structure, we prune latent-to-observed edges whose absolute coefficients are below $0.1$.

Figures~\ref{fig: application}(c)--(g) show the DAGs estimated by each method. In (e) and (f), ``ReLL'' denotes ReLVLiNGAM. As can be seen from these figures, 
the proposed method recovers an observed DAG that is closer to that in Figure~\ref{fig: application}(a) with only one extra edge, $\mathrm{Mek}\to \mathrm{Akt}$, although it estimates a denser $\bm{L}\!\to\!\bm{X}$ structure. 
ParceLiNGAM fails to identify the existence of latent confounders, and outputs redundant edges ${\text{Raf}\to \text{Erk}, \text{Raf}\to \text{Akt},\text{Mek}\to \text{Akt}}$.
RCD correctly concludes that there is no edge between $\text{Raf}$ and $\text{Erk}$, but does not identify any directed edges among the observed variables. 
Even with the pruning, the original ReLVLiNGAM yields an incorrect causal order over $\bm{X}$ and a dense $\bm{L}\!\to\!\bm{X}$ structure.
The modified ReLVLiNGAM yields an $\bm{L}\!\to\!\bm{X}$ structure closer to the reference graph than the original ReLVLiNGAM, but also fails to recover the true ancestral relationships. 

Overall, the proposed method recovers the observed DAG with only one extra edge, $\mathrm{Mek}\to \mathrm{Akt}$. 
Although the original ReLVLiNGAM, the modified ReLVLiNGAM, and the proposed method all infer multiple latent variables, the proposed method and the modified ReLVLiNGAM produce the sparsest $\bm{L}\!\to\!\bm{X}$ structure and most closely match the structure from latent variables to observed variables in Figure~\ref{fig: application}(a).

\section{Conclusion}
\label{sec: conclusion}
In this paper, we proposed a method for recovering the sparsest causal DAG within the observational equivalence class from finite samples. Unlike the original ReLVLiNGAM, the proposed method does not require the local restriction and is applicable to general canonical LvLiNGAMs. Although ReLVLiNGAM consistently estimates the mixing matrix, recovering the sparsest DAG asymptotically still requires an appropriate permutation of its columns. In contrast, the proposed method consistently estimates the sparsest DAG. 

The simulation results and the application to the Sachs protein data demonstrate the superiority of the proposed method over ReLVLiNGAM. Although the modified ReLVLiNGAM can also recover the sparsest causal DAG from finite samples without requiring the local restriction, the proposed method exhibits better finite-sample performance.

Several limitations remain. Since the proposed method relies on the estimation of higher-order cumulants, its performance can deteriorate when the sample size is small or the data are noisy, which may degrade the accuracy of DAG estimation. In addition, the computational cost can still be high when multiple candidate total effects must be examined, although this issue might not be severe when the DAG is sparse.

Improving the accuracy and computational efficiency of the method is an important direction for future work.

\acks{
This work was supported by JST SPRING under Grant Number JPMJSP2110 and JSPS KAKENHI under Grant Numbers 25K15017.
}


\bibliography{acml25}

\appendix
\section{Preservation of the LvLiNGAM Structure under Source Removal}
\label{sec: some thm}
Assume that $X_{s}$ is an observed source of $\mathcal G$.
Let $\mathcal R=[p]\setminus\{s\}$ and fix $u_s\in \{e_s\}\cup \mathrm{Pa}(X_s)$. 
As discussed in Section~\ref{sec: cumulants}, the corresponding DAG might change when swapping $u_{s}$ and $e_{s}$.
Let $\mathcal{G}^{(u_s)}$ denote the DAG in the observational equivalence class of $\mathcal{G}$ that is obtained by swapping $e_s$ and $u_s$. 
For each $j \in \mathcal R$, choose a root $\alpha^{(u_{s})}_{js}$ returned by Proposition~\ref{thm: estimate b} whose associated $k$-th order cumulant is $\kappa^{(k)}(u_{s})$. Then the update rule \eqref{eq: update_vec} is expressed as 
\begin{align*}
    \tilde{\bm X}^{(u_{s})}_{\mathcal R}
    =
    \bm X_{\mathcal R}
    -
    \bm\alpha^{(u_{s})}_{\mathcal R,s}X_s,
    \qquad
    \bm\alpha^{(u_{s})}_{\mathcal R,s}
    =
    \left(
    \alpha^{(u_{s})}_{js}
    \right)_{j\in\mathcal R}.
\end{align*}
\begin{lemma}
\label{lem: induced}
    Under the genericity assumption, $\tilde{\bm X}^{(u_{s})}_{\mathcal R}$ admits an LvLiNGAM representation whose DAG is the induced subgraph of $\mathcal G^{(u_{s})}$ obtained by deleting $X_s$ and all edges linked to $X_s$. 
\end{lemma}
\begin{proof}
Denote the LvLiNGAM representation associated with $\mathcal G^{(u_{s})}$ by
\begin{equation}
    \label{eq: model_us}
    \bm X
    =
    \bm\Lambda^{(u_{s})}\bm L^{(u_{s})}
    +
    \bm B^{(u_{s})}\bm X
    +
    \bm e^{(u_{s})}.
\end{equation}
$\alpha^{(u_s)}_{js}$ is the total effect from $u_s$ to $X_j$. 
Let $\bm{B}^{(u_s)}_{\mathcal R,\mathcal R}$ denote the submatrix of $\bm{B}^{(u_s)}$ with rows and columns indexed by $\mathcal R$.
Let $\bm{b}_{\mathcal R,s}$ denote the column of $\bm{B}^{(u_s)}_{\mathcal R,\mathcal R}$ corresponding to $X_s$.
By the definition of total effects,
\begin{align}
    \bm\alpha^{(u_{s})}_{\mathcal R,s}
    =
    \bm b^{(u_{s})}_{\mathcal R,s}
    +
    \bm B^{(u_{s})}_{\mathcal R,\mathcal R}
    \bm\alpha^{(u_{s})}_{\mathcal R,s}.
    \label{eq: alpha-total-effect-identity}
\end{align}
Let $\bm{\Lambda}^{(u_s)}_{\mathcal R,:}$ denote the submatrix of $\bm{\Lambda}^{(u_s)}$ consisting of the rows indexed by $\mathcal R$ and let $\bm{e}_{\mathcal R}$ denote the subvector of $\bm{e}^{(u_s)}$ consisting of the entries indexed by $\mathcal R$.
Restricting the structural equations to $\mathcal R$, we have
\[
    \bm X_{\mathcal R}
    =
    \bm b^{(u_{s})}_{\mathcal R,s}X_s
    +
    \bm B^{(u_{s})}_{\mathcal R,\mathcal R}\bm X_{\mathcal R}
    +
    \bm\Lambda^{(u_{s})}_{\mathcal R,:}\bm L^{(u_{s})}
    +
    \bm e^{(u_{s})}_{\mathcal R}.
\]
Since
$\bm X_{\mathcal R}=
\tilde{\bm X}^{(u_{s})}_{\mathcal R}+\bm\alpha^{(u_{s})}_{\mathcal R,s}X_s$, 
we obtain 
\begin{align*}
    \tilde{\bm X}^{(u_{s})}_{\mathcal R}
    &=-\bm\alpha^{(u_{s})}_{\mathcal R,s}X_s + \bm b^{(u_{s})}_{\mathcal R,s}X_s +\bm B^{(u_{s})}_{\mathcal R,\mathcal R}(\tilde{\bm X}^{(u_{s})}_{\mathcal R}+\bm\alpha^{(u_{s})}_{\mathcal R,s}X_s)+
    \bm\Lambda^{(u_{s})}_{\mathcal R,:}\bm L^{(u_{s})}
    +
    \bm e^{(u_{s})}_{\mathcal R}\\
    &=
    \bm B^{(u_{s})}_{\mathcal R,\mathcal R}
    \tilde{\bm X}^{(u_{s})}_{\mathcal R}
    +
    \left(
        \bm b^{(u_{s})}_{\mathcal R,s}
        +
        \bm B^{(u_{s})}_{\mathcal R,\mathcal R}
        \bm\alpha^{(u_{s})}_{\mathcal R,s}
        -
        \bm\alpha^{(u_{s})}_{\mathcal R,s}
    \right)X_s 
    +
    \bm\Lambda^{(u_{s})}_{\mathcal R,:}\bm L^{(u_{s})}
    +
    \bm e^{(u_{s})}_{\mathcal R}.
\end{align*}
By \eqref{eq: alpha-total-effect-identity}, the coefficient of $X_s$ is zero. Thus,
\begin{equation}
    \label{eq: tilde x}
    \tilde{\bm X}^{(u_{s})}_{\mathcal R}
    =
    \bm B^{(u_{s})}_{\mathcal R,\mathcal R}
    \tilde{\bm X}^{(u_{s})}_{\mathcal R}
    +
    \bm\Lambda^{(u_{s})}_{\mathcal R,:}\bm L^{(u_{s})}
    +
    \bm e^{(u_{s})}_{\mathcal R}. 
\end{equation}
This is an LvLiNGAM representation over $\bm{X}_{\mathcal R}$, where its coefficient matrix among observed variables is exactly $\bm B^{(u_{s})}_{\mathcal R,\mathcal R}$, and its coefficient matrix from $\bm{L}$ is $\bm\Lambda^{(u_{s})}_{\mathcal R,:}$. 
Hence, the corresponding DAG is obtained from $\mathcal G^{(u_{s})}$ by deleting $X_s$ and all edges linked to $X_s$, namely the induced subgraph of $\mathcal G^{(u_{s})}$ after deleting $X_s$.
\end{proof}
Based on Lemma~\ref{lem: induced}, we have the following Corollary.
\begin{corollary}
\label{cor: invariant B}
    Let $\bm{B}_{\mathcal{R}, \mathcal{R}}$ be the coefficient matrix among $\bm{X} \setminus \{X_{s}\}$ in $\mathcal{G}^{(e_s)}$.
    Then, for each $u_{s} \in \mathrm{Pa}(X_s) \cup \{e_{s}\}$,
    \begin{align*}
        \bm B^{(u_s)}_{\mathcal R,\mathcal R}
        =
        \bm B_{\mathcal R,\mathcal R}.
    \end{align*}
    That is, the observed parent-child relationships among $\bm X_{\mathcal R}$ do not depend on $u_s$.
\end{corollary} 
\begin{proof}
The mixing matrix for $\bm{X}$ is expressed as
\begin{align*}
    (\bm{I}-\bm B)^{-1}=
    \left[\begin{array}{cc}
         1 & \bm 0\\
        \bm{m}^{O}_{\mathcal R,s} & (I-\bm{B}_{\mathcal{R},\mathcal{R}})^{-1}
    \end{array}\right],
\end{align*}
where $\bm{m}^{O}_{\mathcal R,s}$ is the vector of total effects from $X_{s}$ to $\bm{X} \setminus \{X_{s}\}$.
For any $u_s\in\mathrm{Pa}(X_s)\cup\{e_s\}$, the corresponding swap affects only the total effects associated with $X_s$. Thus, $(\bm I-\bm B^{(u_s)})^{-1}$ is obtained from $(\bm I-\bm B)^{-1}$ by replacing the column corresponding to $X_s$ with
$(1, \alpha^{(u_s)}_{\mathcal R,s})^\top$ 
\begin{align*}
(\bm{I}-\bm{B}^{(u_s)})^{-1}
=
\begin{bmatrix}
1 & \bm{0}\\
\bm{\alpha}^{(u_s)}_{\mathcal{R},s} & (\bm{I}-\bm{B}_{\mathcal{R},\mathcal{R}})^{-1}
\end{bmatrix}, 
\end{align*}
where this replacement is trivial when $u_s=e_s$. 

Computing the inverse gives
\begin{align*}
\bm{I}-\bm{B}^{(u_s)}
=
\begin{bmatrix}
1 & \bm{0}\\
-(\bm{I}-\bm{B}_{\mathcal{ R},\mathcal{ R}})\alpha^{(u_s)}_{\mathcal{ R},s} & \bm{I}-\bm{B}_{\mathcal{R},\mathcal{R}}
\end{bmatrix}.
\end{align*}
Thus, we can obtain $\bm{B}^{(u_s)}$
\begin{align*}
    \bm{B}^{(u_s)}
    =
    \begin{bmatrix}
    0 & \bm{0}\\
    (\bm{I}-\bm{B}_{\mathcal{ R},\mathcal{ R}})\alpha^{(u_s)}_{\mathcal{ R},s} & \bm{B}_{\mathcal{R},\mathcal{R}}
\end{bmatrix},
\end{align*}
implying $\bm{B}^{(u_s)}_{\mathcal{R},\mathcal{R}} = \bm{B}_{\mathcal{R},\mathcal{R}}$ from the lower-right block.
\end{proof}
\begin{remark}
The updated model \eqref{eq: tilde x} may contain latent variables in $\mathrm{Pa}(X_s)$ having only one observed child in $\mathcal R$. Such latent variables can be absorbed into the corresponding disturbance term, as in the canonicalization described by \cite{hoyer2008estimation}. 
\end{remark}

\section{Proofs of Theorems in Section~\ref{sec: proposed method}}
\label{sec: proof of thm}
\noindent\noindent\textbf{Proof of Lemma \ref{thm: total effects and edges}}

    Following \cite{drton2011global}, the total effect from an observed source $X_{s}$ to $X_{j}$ is written as
    \begin{align*}
        m^{O}_{js} = b_{js} + \sum_{\pi \in \mathcal{P}(X_{s}, X_{j}) \setminus \{(X_{s}, X_{j})\}} ~ \prod_{l,h: (X_{l}, X_{h}) \in \pi}b_{hl}.
    \end{align*}
    Since every variable in $\widetilde{\mathrm{Ch}}(X_{s})$ is a descendant of $X_{s}$, and every node in $\bm{X}_{\mathrm{open}}$ has ancestors that are either contained in $\{X_{s}\} \cup \widetilde{\mathrm{Ch}}(X_{s})$ or not influenced by $X_{s}$, 
    every directed path from $X_s$ to $X_j$ other than the direct edge $X_s \to X_j$ must first pass through some node $X_i \in \widetilde{\mathrm{Ch}}(X_s)$.
    Thus, 
    \begin{align*}
        m^{O}_{js} = b_{js} + \sum_{i: X_{i} \in \widetilde{\mathrm{Ch}}(X_{s})} b_{is} \sum_{\pi \in \mathcal{P}(X_{i}, X_{j})} ~ \prod_{l,h: (X_{l}, X_{h}) \in \pi}b_{hl} = b_{js} + \sum_{i: X_{i} \in \widetilde{\mathrm{Ch}}(X_{s})} b_{is}~m^{O}_{ji},
    \end{align*}
    which establishes \eqref{eq: total effects and edges}.
    $\hfill\blacksquare$

\noindent\textbf{Proof of Theorem~\ref{cor: total effects and edges}}

By Lemma~\ref{thm: total effects and edges}, the true direct effect satisfies
\begin{align*}
    b_{js}
    =
    m^{O}_{js}
    -
    \sum_{i\in\mathcal I}
    b_{is}m^{O}_{ji.s}.
\end{align*}
Suppose first that $X_s\notin \mathrm{Pa}(X_j)$. Then $b_{js}=0$. Since the true values
$m^{O}_{js}$, $b_{is}$, and $m^{O}_{ji.s}$ are contained in the candidate sets
$\bm m^{O}_{js}$, $\bm b_{is}$, and $\bm m^{O}_{ji.s}$, respectively, there exists a choice
\begin{align*}
    \alpha_{js}\in\bm m^{O}_{js},\qquad
    \beta_{is}\in\bm b_{is},\qquad
    \gamma_{ji.s}\in\bm m^{O}_{ji.s},
\end{align*}
such that $\alpha_{js}$ and $\beta_{is}$, $i\in\mathcal I$, correspond to the same $k$-th order cumulant and 
\begin{align*}
    \alpha_{js}
    -
    \bm\beta_{\mathcal I,s}^{\top}
    \bm\gamma_{j,\mathcal I.s}
    =
    m^{O}_{js}
    -
    \sum_{i\in\mathcal I}
    b_{is}m^{O}_{ji.s}
    =
    0.
\end{align*}
Thus, \eqref{eq: total effects and edges generalized} holds.

Conversely, suppose that $X_s\in \mathrm{Pa}(X_j)$. Then $b_{js}\neq0$. Under the genericity assumption, no choice of candidates from
$\bm m^{O}_{js}$, $\bm b_{is}$, and $\bm m^{O}_{ji.s}$, with
$\alpha_{js}$ and $\beta_{is}$ corresponding to the same $k$-th order cumulant, can satisfy
\begin{align*}
    \alpha_{js}
    -
    \bm\beta_{\mathcal I,s}^{\top}
    \bm\gamma_{j,\mathcal I.s}
    =
    0.
\end{align*}
Hence, \eqref{eq: total effects and edges generalized} fails generically.

Therefore, \eqref{eq: total effects and edges generalized} holds if and only if
$X_s\notin \mathrm{Pa}(X_j)$.
$\hfill\blacksquare$

\noindent\textbf{Proof of Theorem \ref{thm: unchanged}}

Immediate from Lemma~\ref{lem: induced} and Corollary~\ref{cor: invariant B}.
$\hfill\blacksquare$



\noindent\textbf{Proof of Theorem \ref{thm: unchanged OL}}

Let $\kappa^{(k)}(u_s)$ denote the cumulant corresponding to
$\bm{\alpha}^{*}_{\mathcal R,s}$.
By \eqref{eq: model_us}, the corresponding latent-to-observed coefficient matrix is
$\Lambda^{(u_s)}$.
Since $\bm{\alpha}^{*}_{\mathcal R,s}$ is chosen to minimize the number of latent parents, the support of $\Lambda^{(u_s)}$ is the sparsest among all latent-to-observed structures in the observational equivalence class. Therefore, it coincides with the structure of $\mathcal G^{OL}$.

By \eqref{eq: tilde x}, the latent-to-observed coefficient matrix for the model of
$\tilde{\bm X}_{\mathcal R}$ is
$\Lambda^{(u_s)}_{\mathcal R,:}$.
Hence, its support coincides with that of the induced subgraph of
$\mathcal G^{OL}$ on $\bm V\setminus\{X_s\}$ obtained by absorbing every latent variable having only one observed child into the disturbance of that child.
$\hfill\blacksquare$

\noindent\textbf{Proof of Theorem~\ref{thm: identify mixing matrix}}

Immediate from Theorem~\ref{thm: unchanged} and Theorem~\ref{thm: unchanged OL}.
$\hfill\blacksquare$

\section{A Modification of the Original ReLVLiNGAM}
\label{sec: simple improvement}

As mentioned earlier, the origianl ReLVLiNGAM, proposed by \cite{schkoda2024causal}, cannot be applied when the local restriction is violated. In this section, we modify the original ReLVLiNGAM so that it can be applied even when the local restriction is violated.

Like the proposed method, ReLVLiNGAM is a top-down algorithm that recursively estimates total effects. Here, let $X_1$ be an observed source and let $X_i \in \mathrm{Des}(X_1)$. The update rule for $X_i$ in ReLVLiNGAM is given by
\begin{align}
\label{eq: updated Xj}
X_i
\gets
X_i
-
m_{i1}^{O}e_1
-
\sum_{h:L_h\in\mathrm{Pa}(X_1)}
m_{ih}^{OL}L_h.
\end{align}
Although this update cannot be computed directly because neither the disturbances nor the latent variables are observed, the higher-order cumulants of the updated variables satisfy
\begin{align*}
c^{(k)}_{i_{1},\ldots,i_{k}}
\gets
c^{(k)}_{i_{1},\ldots,i_{k}}
-
m^{O}_{i_{1}1}\cdots m^{O}_{i_{k}1}\kappa^{(k)}(e_1)
-
\sum_{h:L_h\in\mathrm{Pa}(X_1)}
m^{OL}_{i_{1}h}\cdots m^{OL}_{i_{k}h}\kappa^{(k)}(L_h).
\end{align*}
Hence, once the higher-order cumulants of the disturbances and latent variables have been estimated, the higher-order cumulants after the update can be computed without explicitly updating the observed variables. 

For a source node $X_1$, \cite{schkoda2024causal} proposed estimating
$$
(\kappa^{(k)}(e_1),\kappa^{(k)}(L_1),\ldots,\kappa^{(k)}(L_q))
$$ 
by solving the linear system
\[
\begin{bmatrix}
1 & 1 & \cdots & 1\\
m_{21}^O & m_{21}^{OL} & \cdots & m_{2q}^{OL}\\
\vdots & \vdots & \ddots & \vdots\\
m_{p1}^O & m_{p1}^{OL} & \cdots & m_{pq}^{OL}
\end{bmatrix}
\begin{bmatrix}
\kappa^{(k)}(e_1)\\
\kappa^{(k)}(L_1)\\
\vdots\\
\kappa^{(k)}(L_q)
\end{bmatrix}
=
\begin{bmatrix}
c^{(k)}_{1\cdots11}\\
c^{(k)}_{1\cdots12}\\
\vdots\\
c^{(k)}_{1\cdots1p}
\end{bmatrix}.
\]
However, when the local restriction is violated, this linear system becomes underdetermined and therefore cannot be solved. By contrast, Proposition~\ref{lem: estimate e cumulant} enables the estimation of the higher-order cumulants of the disturbance and the latent confounders even without the local restriction. In \cite{schkoda2024causal}, Proposition~\ref{lem: estimate e cumulant} is used solely to establish the correspondence between total effects and higher-order cumulants. Once this correspondence has been identified, the higher-order cumulants can in turn be estimated, making it possible to compute the cumulant update above. Consequently, the top-down update procedure remains applicable even when the local restriction is violated. 

The overall procedure of the modified ReLVLiNGAM is as follows.
At each iteration, it first constructs $A^{(k_1,k_2)}_{j,i}$ using sufficiently high-order cumulants and applies Proposition~\ref{thm: latent number} to estimate the number of latent confounders and the ancestral relationships among the observed variables.
Based on the estimated ancestral relationships, the observed source nodes are identified.
For each observed source, the candidate total effects from the source to the remaining observed variables are estimated by extending Proposition~\ref{thm: estimate b} to the present setting.
Next, Proposition~\ref{lem: estimate e cumulant} is used to estimate the higher-order cumulants of the disturbances and latent variables whenever they are identifiable.
These estimates are then used to update the cumulants of the residualized variables.
The same procedure is repeated recursively until all columns of the mixing matrix corresponding to the total effects have been estimated.

To avoid using unnecessarily high-order cumulants, we choose $k_1$ and $k_2$ to be the smallest values that satisfy the requirements of Proposition~\ref{thm: latent number}. 
For each descendant $X_i$ of an observed source $X_1$, let $k_{i,1}$ denote the smallest order for which the linear system in Proposition~\ref{lem: estimate e cumulant} is generically solvable.
$k_{i,1}$ can be estimated by increasing $k$ from $2$ and checking whether the corresponding linear system \eqref{eq: system} is generically solvable. Applying the same procedure to every descendant of $X_1$ yields the corresponding values $k_{i,1}$.

Now suppose that $X_i$ is the next observed source after removing $X_1$.
For another descendant $X_j$ of $X_1$, let
\begin{align}
    \label{eq: k1k2}
    k_1 = \max\{\ell+2, k_{i,1}, k_{j,1}\},\qquad
    k_2 = k_1+\left\lceil \frac{-3+\sqrt{8\ell+17}}{2} \right\rceil,
\end{align}
and construct $A^{(k_1,k_2)}_{i,j}$ as in \eqref{eq: a matrix}.
Starting from $\ell=0$, $\ell$ is estimated by increasing $\ell$ until the rank condition in Proposition~\ref{thm: latent number} is satisfied. 



Similarly, define ${\tilde{A}}^{(k_1,k_2)}_{j,i}$ by appending the row $(1,m,\ldots,m^{k_1-1})$ on top of ${A}^{(k_{1}, k_{2})}_{j, i}$ and then retaining the first $\ell+2$ columns. 
In this case, an analogous result to Proposition~\ref{thm: estimate b} holds for $(k_1,k_2)$ defined in \eqref{eq: k1k2}.
\begin{theorem}
\label{thm: estimate b revised}
    Consider the determinant of an $(\ell+2)\times (\ell+2)$ minor of ${\tilde{A}}^{(k_{1}, k_{2})}_{j, i}$ that contains the first row and treat it as a polynomial in $m$. 
    Then, the roots of this polynomial are $m^{O}_{ji}, m^{OL'}_{j1}, \cdots, m^{OL'}_{j\ell}$.
\end{theorem}
\begin{proof}
Similarly to the proof of Theorem 4 in~\cite{schkoda2024causal}, we can show that the determinant of a $(\ell+2)\times(\ell+2)$ minor of 
$\widetilde{A}^{(k_1,k_2)}_{j,i}$ containing the first row is generically not the zero polynomial in $m$.
However, when $m$ in the first row is set to any value in 
$\{m^{O}_{ji}, m^{OL'}_{j1}, \ldots, m^{OL'}_{j\ell}\}$, the determinant of this minor vanishes.
Since this determinant is a nonzero polynomial of degree at most $\ell+1$ in $m$, and these $\ell+1$ values are generically distinct, they are exactly the roots of the polynomial, which completes the proof.
\end{proof}
After obtaining $\{m^{O}_{ji}, m^{OL'}_{j1}, \cdots, m^{OL'}_{j\ell}\}$, we choose the lowest order $k_{j,i}$ for which \eqref{eq: system} in Proposition~\ref{lem: estimate e cumulant} is solvable. 
Then, the $k_{j,i}$-th and higher-order cumulants of the residualized variables can be updated, and the procedure can proceed to the next iteration.

Unlike the proposed method, the modified ReLVLiNGAM still relies on updates based on estimated disturbance cumulants. 
Therefore, estimation errors in higher-order cumulants may still propagate to subsequent iterations.

\section{Illustration of the Proposed Algorithm on the DAG in Figure~\ref{fig: equivalence class of IV graph}(a)}
\label{sec: example}
The LvLiNGAM for the model in Figure~\ref{fig: equivalence class of IV graph}(a) is expressed as
\begin{align*}
    X_{1} &= L_{1} + L_{2} + e_{1},\\
    X_{2} &= \lambda_{21}L_{1} + \lambda_{22}L_{2} + L_{3} + b_{21}X_{1}+ e_{2}\\
          &= (b_{21}+\lambda_{21})L_{1} + (b_{21}+\lambda_{22})L_{2} + L_{3} + b_{21}e_{1} + e_{2},\\
    X_{3} &= \lambda_{33}L_{3} + b_{32}X_{2} + e_{3}\\
          &= b_{32}(b_{21}+\lambda_{21})L_{1} + b_{32}(b_{21}+\lambda_{22})L_{2} + (b_{32}+\lambda_{33})L_{3}
             + b_{21}b_{32}e_{1} + b_{32}e_{2} + e_{3}.
\end{align*}

By Propositions~\ref{thm: latent number}--\ref{lem: estimate e cumulant}, the proposed method identifies the observed source, the relevant total effects, and the associated disturbance cumulants. 
Here, $X_{1}$ is the unique observed source, and $\widetilde{\mathrm{Ch}}(X_{1})=\{X_{2}\}$.
Hence,
\begin{align*}
    \bm{b}_{21} &= \{b_{21},\ (b_{21}+\lambda_{21}),\ (b_{21}+\lambda_{22})\},\\
    \bm{m}^{O}_{31} &= \{b_{21}b_{32},\ b_{32}(b_{21}+\lambda_{21}),\ b_{32}(b_{21}+\lambda_{22})\},
\end{align*}
together with the correspondences between cumulants and total effects 
\begin{align}
\label{eq: pairs}
    \kappa^{(k)}(e_{1})\ &:\ \big\{b_{21},\ b_{21}b_{32}\big\}, \notag\\
    \kappa^{(k)}(L_{1})\ &:\ \big\{(b_{21}+\lambda_{21}),\ b_{32}(b_{21}+\lambda_{21})\big\}, \\
    \kappa^{(k)}(L_{2})\ &:\ \big\{(b_{21}+\lambda_{22}),\ b_{32}(b_{21}+\lambda_{22})\big\}, \notag
\end{align}
where \eqref{eq: system} is solvable at the order $k$.

Among the candidates satisfying \eqref{eq: total effects and edges generalized}, we choose
\[
\alpha_{31}=b_{32}(b_{21}+\lambda_{21}), \qquad
\beta_{21}=b_{21}+\lambda_{21},
\]
which correspond to $\kappa^{(k)}(L_1)$.
Then
\begin{align*}
    \tilde{X}_2 = X_{2}-\beta_{21}X_{1}
      &
       = \big((\lambda_{22}-\lambda_{21})L_{2}-\lambda_{21}e_{1}\big)+L_{3}+e_{2},\\
       \tilde{X}_3=
    X_{3}-\alpha_{31}X_{1}
      &
      = b_{32}\big((\lambda_{22}-\lambda_{21})L_{2}-\lambda_{21}e_{1}\big)+(b_{32}+\lambda_{33})L_{3}+b_{32}e_{2}+e_{3}\\
        &= b_{32} \tilde{X}_2 + \lambda_{33}L_{3}+e_{3}
\end{align*}
By Proposition~\ref{thm: latent number}, there is only one confounder between $\tilde{X}_2$ and $\tilde{X}_3$.
Furthermore, by Proposition~\ref{thm: estimate b}, 
we have
\[
\bm{m}^{O}_{32.1}=\{b_{32},\ (b_{32}+\lambda_{33})\}.
\]
Choosing $\gamma_{32.1}=b_{32}$ gives $\alpha_{31}-\beta_{21}\gamma_{32.1}=0$, and Theorem~\ref{cor: total effects and edges} concludes $X_{1}\notin \mathrm{Pa}(X_{3})$.

In fact, every aligned pair in \eqref{eq: pairs} can satisfy \eqref{eq: total effects and edges generalized} and the number of latent confounders between remaining variables is always one.
For instance, suppose that the method selects
\[
\alpha^*_{21}=b_{21}+\lambda_{22}, 
\qquad
\alpha^*_{31}=b_{32}(b_{21}+\lambda_{22})
\] 
to update $X_{2}$ and $X_{3}$. 
The corresponding total effect from $X_{2}$ to $X_{3}$ is also $b_{32}$, and \eqref{eq: total effects and edges generalized} still holds.
Since there are two unselected aligned groups of total effects, they correspond to two distinct latent variables and hence to two distinct total-effect columns in the mixing matrix.
Then, the mixing matrices $\widehat{\bm{M}}$ and $\widehat{\bm{M}}_{\mathrm{latent}}$ in Algorithm~\ref{alg: proposed} are updated to
\begin{align*}
    \widehat{\bm{M}} \gets 
    \left[
    \begin{array}{ccc}
       1  & 0 & 0\\
       (b_{21}+\lambda_{22})  & 1 & 0\\
       b_{32}(b_{21}+\lambda_{22}) & b_{32} & 1
    \end{array}
    \right], \quad 
    \widehat{\bm{M}}_{\mathrm{latent}} \gets 
    \left[
    \begin{array}{cc}
       1  & 1\\
       b_{21}  & (b_{21}+\lambda_{21})\\
       b_{21}b_{32} & b_{32}(b_{21}+\lambda_{21})
    \end{array}
    \right].
\end{align*}
In this case, $\tilde{X}_2$ and $\tilde{X}_3$ are
\begin{align*}
    \tilde{X}_{2} &= \big((\lambda_{21}-\lambda_{22})L_{1}-\lambda_{22}e_{1}+e_{2}\big)+L_{3},\\
    \tilde{X}_{3} &= b_{32}\big((\lambda_{21}-\lambda_{22})L_{1}-\lambda_{22}e_{1}+e_{2}\big)+(b_{32}+\lambda_{33})L_{3}+e_{3}.
\end{align*}

Applying Proposition~\ref{thm: latent number} again yields $X_{2}$ as the current observed source and $X_{3}$ as its unique descendant, hence $X_{2}\in \mathrm{Pa}(X_{3})$. Applying Proposition~\ref{lem: estimate e cumulant} to $(X_{2},X_{3})$ gives candidates $b_{32}$ and $(b_{32}+\lambda_{33})$. Since the total effect from $X_{2}$ to $X_{3}$ is already determined as $b_{32}$, the remaining value $(b_{32}+\lambda_{33})$ corresponds to a confounder between $X_{2}$ and $X_{3}$. Thus, 
\begin{align*}
    \widehat{\bm{M}}_{\mathrm{latent}} \gets
    \left[
    \begin{array}{ccc}
       1  & 1 & 0\\
       b_{21}  & (b_{21}+\lambda_{21}) & 1\\
       b_{21}b_{32} & b_{32}(b_{21}+\lambda_{21}) & (b_{32}+\lambda_{33})
    \end{array}
    \right].
\end{align*}
Consequently,
\begin{align*}
    \widehat{\bm{B}}
    &= \bm{I}-\widehat{\bm{M}}^{-1}
    = \left[
    \begin{array}{ccc}
       0  & 0 & 0\\
       b_{21}+\lambda_{22}  & 0 & 0\\
       0 & b_{32} & 0
    \end{array}
    \right], \\
    \widehat{\bm{\Lambda}}
    &= \widehat{\bm{M}}^{-1}\widehat{\bm{M}}_{\mathrm{latent}}
    = \left[
    \begin{array}{ccc}
       1  & 1 & 0\\
       -\lambda_{22}  & \lambda_{21}-\lambda_{22} & 1\\
       0 & 0 & \lambda_{33}
    \end{array}
    \right].
\end{align*}
Let the columns of $\widehat{\bm{\Lambda}}$ correspond to $(L_{1},L_{2},L_{3})$. Then the recovered parent sets are
\begin{align*}
\mathrm{Pa}(X_{1})=\{L_{1},L_{2}\},\quad
\mathrm{Pa}(X_{2})=\{L_{1},L_{2},L_{3}, X_{1}\},\quad
\mathrm{Pa}(X_{3})=\{L_{3},X_{2}\}.
\end{align*}

In finite samples, entries that are theoretically zero in $\hat{\bm B}$ and $\hat{\bm\Lambda}$ may be estimated as nonzero.
For $\hat{\bm B}$, whenever no parent-child relationship is estimated between two observed variables, the corresponding entry of $\hat{\bm B}$ is set to zero.
For $\hat{\bm\Lambda}$, we prune the matrix by setting entries whose absolute values are below predefined thresholds to zero.

\end{document}